\documentclass[]{fairmeta}
\usepackage{lineno}
\usepackage{calc}
\usepackage{amsmath}
\usepackage{caption}
\usepackage{multirow}
\usepackage{amsfonts}
\usepackage{xspace}
\usepackage[font=footnotesize]{subcaption}
\usepackage{booktabs}
\usepackage{pdfpages}
\usepackage{xcolor}
\usepackage{tabularx}
\usepackage{bbm}
\usepackage{graphicx}
\usepackage{algorithm}
\usepackage{algorithmic}
\usepackage{enumitem}
\usepackage{pifont}

\usepackage{amsmath}
\usepackage{amssymb}
\usepackage{mathtools}
\usepackage{amsthm}
\usepackage{xspace}
\usepackage{algorithm}

\theoremstyle{plain}
\newtheorem{theorem}{Theorem}[section]
\newtheorem{proposition}[theorem]{Proposition}

\newtheorem{corollary}[theorem]{Corollary}
\theoremstyle{definition}

\theoremstyle{remark}
\newtheorem{remark}[theorem]{Remark}

\newcommand{\ours}{\textsc{EBPO}\xspace}

\title{EBPO: Empirical Bayes Shrinkage for Stabilizing Group-Relative Policy Optimization}

\author[*]{Kevin Han}
\author[*]{Yuhang Zhou}
\author{Mingze Gao}
\author{Gedi Zhou}
\author{Serena Li}
\author{Abhishek Kumar}
\author{Xiangjun Fan}
\author[\dagger]{Weiwei Li}
\author[\dagger]{Lizhu Zhang}

\affiliation{Meta AI}

\contribution[*]{Equal contribution.}
\contribution[\dagger]{Joint last author.}

\abstract{
Reinforcement Learning with Verifiable Rewards (RLVR) has proven effective for enhancing the reasoning capabilities of Large Language Models (LLMs). However, dominant approaches like Group Relative Policy Optimization (GRPO) face critical stability challenges: they suffer from high estimator variance under computational constraints (small group sizes) and vanishing gradient signals in saturated failure regimes where all responses yield identical zero rewards. To address this, we propose Empirical Bayes Policy Optimization (EBPO), a novel framework that regularizes local group-based baselines by borrowing strength from the policy’s accumulated global statistics. Instead of estimating baselines in isolation, EBPO employs a shrinkage estimator that dynamically balances local group statistics with a global prior updated via Welford’s online algorithm. Theoretically, we demonstrate that EBPO guarantees strictly lower Mean Squared Error (MSE), bounded entropy decay, and non-vanishing penalty signals in failure scenarios compared to GRPO. Empirically, EBPO consistently outperforms GRPO and other established baselines across diverse benchmarks, including AIME and OlympiadBench. Notably, EBPO exhibits superior training stability, achieving high-performance gains even with small group sizes, and benefits significantly from difficulty-stratified curriculum learning.

}

\date{\today}
\correspondence{Kevin Han, Yuhang Zhou and Lizhu Zhang at \email{\{kevinwh, zyhang, lizhu\}@meta.com}\\}

\begin{document}

\maketitle

\section{Introduction}

Post-training for Large Language Models (LLMs) has increasingly shifted toward reinforcing reasoning using verifiable signals. Reinforcement Learning with Verifiable Rewards (RLVR) has emerged as a stable and reproducible paradigm for tasks such as mathematical reasoning and code generation, where objective correctness can serve as a ground-truth reward signal \citep{shao2024deepseekmath, guo2025deepseek}. To avoid the computational overhead of training separate value networks, recent advancements have coalesced around group-based methods, most notably Group Relative Policy Optimization (GRPO) \citep{shao2024deepseekmath}. By normalizing rewards within a sampled group of outputs for a given prompt, GRPO provides a computationally efficient baseline that has demonstrated significant improvements in model reliability \citep{shao2024deepseekmath, yu2025dapo}.


Despite its success, GRPO faces inherent limitations regarding the stability of its advantage estimator. Because the baseline is derived solely from the local group mean, it is susceptible to high variance when the group size ($G$) is small \citep{yu2025dapo}. Furthermore, GRPO lacks robustness in ``saturated'' regimes: if a model fails all attempts for a difficult prompt (all rewards are zero), the relative advantage vanishes, resulting in a null gradient and a wasted training step \citep{liu2025understanding}. While recent approaches like DAPO \citep{yu2025dapo} attempt to mitigate this volatility by filtering or reweighting unstable updates, they do so by effectively discarding the problematic examples, leading to unavoidable data waste. Consequently, standard GRPO often necessitates larger group sizes to suppress noise without losing data, thereby drastically increasing the computational cost of training \citep{shao2024deepseekmath}.

In this work, we introduce \textbf{Empirical Bayes Policy Optimization (EBPO)}, which reframes advantage estimation through the lens of Empirical Bayes (EB) inference \citep{robbins1992empirical}. Classical EB theory suggests that when estimating parameters for parallel tasks, one can reduce estimation error by assuming those parameters share a common underlying distribution \citep{robbins1992empirical}. Applying this to RLVR, we postulate that the latent success probability of a prompt is drawn from a global distribution characterized by the policy's historical performance.

EBPO replaces the purely local GRPO baseline with a shrinkage estimator that pulls the noisy group mean toward a global mean, $\mu_{glob}$. The degree of this shrinkage is determined dynamically by the ratio of within-group variance to between-group variance. This formulation allows EBPO to distinguish between failing a globally ``hard'' task (consistent with the prior) versus failing an ``easy'' task (deviating from the prior), assigning adaptive penalty signals accordingly. To ensure scalability, we estimate these global priors dynamically using Welford's online algorithm \citep{welford1962note}. 



We validate EBPO through both theoretical analysis and extensive empirical evaluation. Theoretically, we prove that the EBPO baseline achieves a strictly lower Mean Squared Error (MSE) than the standard sample mean used in GRPO and provides informative, non-zero gradients even when group rewards are saturated. Empirically, we evaluate EBPO on various datasets with different scales of base models. Our results show that EBPO consistently outperforms baselines on challenging benchmarks such as AIME-2024 \citep{li2024numinamath} and OlympiadBench \citep{he2024olympiadbench}. Furthermore, we demonstrate that EBPO is highly sample-efficient, outperforming GRPO by over 11\% in extremely resource-constrained settings ($G = 8$), while effectively leveraging curriculum learning via difficulty-stratified sampling.

Our contributions are summarized as follows:
\begin{itemize}
    \item \textbf{Algorithmic Framework:} We propose EBPO, which integrates online Empirical Bayes estimation into the GRPO framework to stabilize advantage computation.
    \item \textbf{Theoretical Guarantees:} We provide proofs that EBPO resolves the vanishing gradient problem in saturated groups and strictly reduces the variance of the baseline estimator compared to GRPO.
    \item \textbf{State-of-the-Art Performance:} We demonstrate that EBPO achieves superior performance across multiple models and benchmarks, particularly when combined with difficulty-based curriculum learning.
\end{itemize}
\section{Related Work}
\label{sec:related}

\paragraph{Reinforcement Learning with Verifiable Rewards.}

Previous work has shown that RLVR is effective in enhancing reasoning and factual correctness in LLMs \citep{shao2024deepseekmath, guo2025deepseek}. By leveraging automatically checkable signals, RLVR enables stable and reproducible optimization for reasoning tasks. A variety of algorithmic variants were proposed, including GRPO-based approaches that integrate verifiable rewards into group-based policy optimization frameworks \citep{shao2024deepseekmath, liu2025understanding, yu2025dapo, cui2025entropy, zheng2025group, zhou2025mixture}. These methods demonstrate that verifiable reward signals can significantly improve model reliability. 

Among these GRPO variants, a broad range of studies focus on improving the effectiveness and efficiency of RLVR training, such as by proposing novel learning objectives \citep{liu2025understanding, zhou2025disco, yu2025dapo, yue2025vapo, liu2025uniform}, introducing entropy-based mechanisms to encourage broader exploration \citep{cui2025entropy, wang2025beyond, zhang2025edge, le2025no, yang2025let}, or designing more robust reward functions \citep{jia2025autorubric, li2025semantically, xiong2026token}. However, few have considered the estimation of the mean or standard deviation of rewards during advantage computation, which can substantially influence group-level differences. Our work addresses this gap by employing an Empirical Bayes method to more robustly estimate the mean and standard deviation of group rewards, thereby stabilizing the training dynamics of GRPO.

\citet{zeng2025shrinking} also address the high-variance bottleneck in GRPO, using Stein’s Paradox to theoretically guarantee that their shrinkage-based baseline reduces variance. Unlike their approach, our \ours framework uses Empirical Bayes inference to dynamically control shrinkage based on variance ratios, which is crucial for avoiding vanishing gradients in challenging prompts. Both works highlight the importance of global-local baseline interpolation for stable RL with verifiable rewards.

\paragraph{Empirical Bayes Estimation}

Empirical Bayes provides a framework to handle estimation in the face of parallel, noisy data. It is a powerful methodology that blends the data-driven nature of frequentist statistics with the hierarchical modeling structure of Bayesian inference, offering a principled approach to regularization and robust estimation. Instead of specifying the prior hyperparameters in advance, Empirical Bayes uses the observed data to estimate them \citep{robbins1992empirical}. Later, statisticians found that shrinkage estimators like the James-Stein estimator \citep{james1961estimation} can be viewed as Empirical Bayes estimators with a Gaussian prior. Empirical Bayes methods have been widely adopted for hyperparameter tuning in Bayesian reinforcement learning \citep{reisinger2008online} and prior estimation in multi-task meta-learning \citep{hu2020empirical, grant2018recasting}. Additionally, these methods facilitate uncertainty estimation in deep neural networks \citep{krishnan2020specifying} and the development of Bayesian dialogue agents through data-driven parameter initialization \citep{lee-etal-2023-empirical}.

While Bayesian principles have been increasingly applied to LLMs, their utility has been largely confined to evaluation frameworks and black-box optimization rather than online policy training. For instance, \citet{liu2024large} integrates LLMs into Bayesian Optimization to enhance surrogate modeling and candidate sampling for hyperparameter tuning. \citet{xiao2025confidence} and \citet{hariri2025don} propose a Bayesian approach to estimate model capabilities in the evaluation process. However, these approaches operate primarily on static evaluation or auxiliary optimization tasks. To the best of our knowledge, EBPO is the first framework to integrate Empirical Bayes directly into the online optimization loop of GRPO, utilizing shrinkage estimators to stabilize gradient variance during active training.

\section{Methodology}

\subsection{Preliminaries}

\paragraph{Reinforcement Learning from Verifiable Rewards (RLVR)}
RLVR provides a memory-efficient framework for optimizing reasoning policies $\pi_\theta$ in domains where the correctness of a response $o$ can be objectively determined by a verifier \citep{shao2024deepseekmath}. Unlike traditional actor-critic methods that require a learned value network, RLVR utilizes group-relative advantage estimation to provide a baseline for policy updates.

For a given prompt $q$, the policy $\pi_\theta$ samples a group of $G$ independent trajectories $\{o_1, o_2, \dots, o_G\}$. Each trajectory is assigned a verified reward $r_i \in \{0, 1\}$ (or a shaped scalar reward) based on its correctness. The advantage $\hat{A}_i$ for each response is then computed by standardizing the reward within its group: $\hat{A}_i = \frac{r_i - \mu_{\text{group}}}{\sigma_{\text{group}} + \epsilon}$, 
where $\mu_{\text{group}}$ and $\sigma_{\text{group}}$ are the empirical mean and standard deviation of the rewards within the group:
\begin{equation*}
    \mu_{\text{group}} = \frac{1}{G} \sum_{j=1}^G r_j, \quad \sigma_{\text{group}} = \sqrt{\frac{1}{G-1} \sum_{j=1}^G (r_j - \mu_{\text{group}})^2}
\end{equation*}

The policy parameters $\theta$ are updated by maximizing the clipped surrogate objective:

\begin{equation}
    \mathcal{J}_{\text{RLVR}}(\theta) = \mathbb{E}_{q \sim \mathcal{D}} \bigg[ \frac{1}{G} \sum_{i=1}^G \min \Big( \rho_i(\theta) \hat{A}_i, \text{clip}(\rho_i(\theta), 1-\epsilon, 1+\epsilon) \hat{A}_i \Big) \bigg]
\end{equation}
where $\rho_i(\theta) = \frac{\pi_\theta(o_i|q)}{\pi_{\theta_{\text{old}}}(o_i|q)}$ is the importance sampling ratio. While RLVR eliminates the computational cost of a critic model, the estimator $\mu_{\text{group}}$ suffers from high variance when $G$ is small. This instability often leads to noisy gradients, particularly in mathematical reasoning where the reward distribution is sparse.

\paragraph{Empirical Bayes and Simultaneous Estimation}

In classical Bayesian inference, a prior distribution is specified before any data is observed. In Empirical Bayes (EB), the prior is instead estimated from the data itself \citep{robbins1992empirical}. Consider the general problem of simultaneously estimating a set of $M$ related parameters $\theta_1, \dots, \theta_M$ based on independent observations $y_1, \dots, y_M$, where each $y_m$ serves as a noisy estimate of $\theta_m$. Rather than estimating each parameter in isolation---which often yields high variance when data is sparse---EB assumes that the parameters are exchangeable and drawn from a common global prior distribution:
$\theta_m \sim \mathcal{N}(\mu, \tau^2)$,
where $\mu$ is the global mean and $\tau^2$ is the between-parameter variance (or prior variance). By estimating these global hyperparameters $(\mu, \tau^2)$ from the marginal distribution of the entire dataset, we can construct a posterior estimate for each individual $\theta_m$. This results in a shrinkage estimator, which ``pulls'' individual, noisy observations toward the global mean $\mu$. 


\subsection{Empirical Bayes Policy Optimization (EBPO)}
A primary challenge in post-training LLMs on reasoning tasks, such as mathematics or formal logic, is the \textit{sparsity and high variance of rewards} \citep{uesato2022solving, zelikman2022star, lightman2023lets, shao2024deepseekmath, yu2025dapo}. In GRPO methods, the advantage is computed by normalizing rewards within a group of completions for a single prompt. However, when a prompt is excessively difficult, the model may fail all attempts (i.e., all rewards are zero), resulting in a null gradient. Conversely, for trivial prompts, the model might succeed in all attempts, again providing no relative signal for improvement. 

We propose \textbf{Empirical Bayes Policy Optimization (EBPO)}, which regularizes the local group-based baseline by ``borrowing strength'' from the global performance distribution of the policy. We assume that for any prompt $q$, the true latent success probability $\theta_q$ is drawn from a global distribution: $\theta_q \sim \mathcal{N}(\mu_{\text{glob}}, \tau^2)$, where $\mu_{\text{glob}}$ represents the global average success rate of the policy and $\tau^2$ captures the variance in difficulty across different prompts. While Beta-Binomial models are natural for binary rewards, we employ a Gaussian approximation to facilitate tractable online inference, which yields a closed-form linear shrinkage estimator that avoids the numerical hyperparameter optimization required by Beta priors, ensuring computational efficiency for online training. Following the Empirical Bayes paradigm, we do not fix these hyperparameters \textit{a priori} but estimate $\mu_{\text{glob}}$ and $\tau^2$ dynamically from the data across all prompts in the training history. This allows us to compute a shrinkage estimator that pulls the noisy local mean $\mu_{\text{group}}$ toward the global policy average (Figure~\ref{fig:ebpo_overview}).

\begin{figure*}[t]
    \centering
    \includegraphics[width=0.95\textwidth]{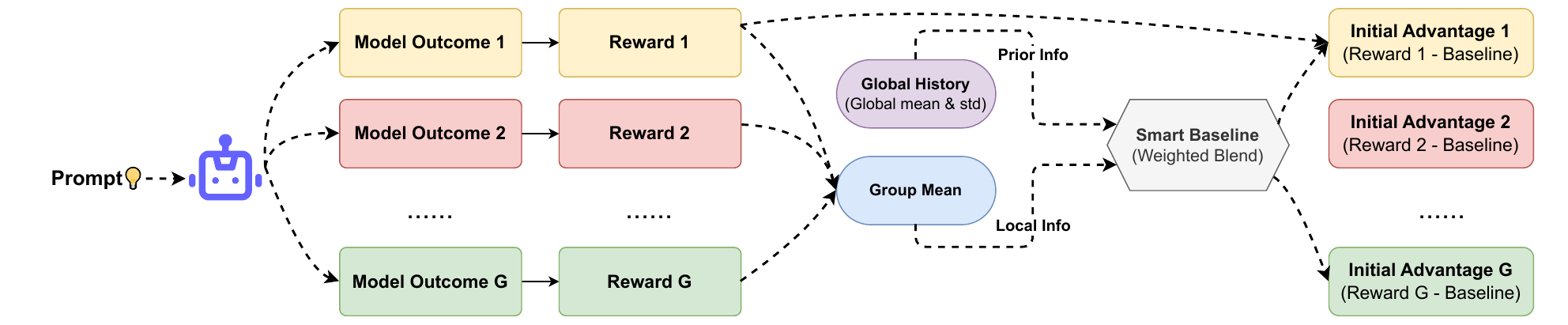}
    \caption{\textbf{Overview of the EBPO Framework.} Unlike standard GRPO which relies solely on the local group mean, EBPO computes a ``Smart Baseline'' (Shrinkage Estimator) by blending the noisy local group mean with a stable global prior (updated via global history). This allows for informative advantage estimates even in small groups or saturated failure regimes.}
    \label{fig:ebpo_overview}
\end{figure*}

Let $\{q_1, \dots, q_M\}$ be a batch of prompts. For each prompt $q_m$, we sample $G$ responses $\{o_{m,1}, \dots, o_{m,G}\}$ and obtain rewards $r_{m,i}$. We define the local group mean as $\bar{r}_m = \frac{1}{G} \sum_{i=1}^G r_{m,i}$. EBPO models the relationship between local performance and global policy capability using two levels of variance:
\begin{enumerate}
   \item \textbf{Within-group variance ($\sigma^2$)}: The variance of individual rewards for a given prompt across all generations. We estimate this with $\sigma^2 = \text{Var}(r_{m,i})$.
   \item \textbf{Between-group variance ($\tau^2$)}: The variance of the true latent success rates across different prompts. We estimate this directly using the variance of the observed prompt means $\tau^2 = \text{Var}(\bar{r}_m)$.
\end{enumerate}
\begin{remark}[Online Stability vs. Bias for Between Group Variances] \label{remark:between_group_var}
    We note that strictly speaking we would define the prior variance as $\tau_{latent}^2 = \text{Var}(\bar{r}_m) - \sigma^2/G$ to remove the sampling noise. However, in online training settings, this subtraction can lead to negative variance estimates. By using the raw $\text{Var}(\bar{r}_m)$ as a proxy for $\tau^2$, we effectively inflate the denominator, resulting in a \textit{conservative} shrinkage estimator. This ensures the shrinkage factor $\mathcal{S}_q$ remains strictly in $[0, 1]$ and favors the local group mean when uncertainty is high, providing a safe "soft" regularization that prevents policy collapse without over-constraining the model.
\end{remark}
The EBPO Baseline $V_q^{EB}$ for a prompt $q$ is defined as the posterior mean of the reward:
\begin{equation}
    V_q^{EB} = (1 - \mathcal{S}_q) \mu_{\text{group}} + \mathcal{S}_q \mu_{\text{glob}}
\end{equation}
where the shrinkage factor $\mathcal{S}_q \in [0, 1]$ is determined by the ratio of variances:
\begin{equation}
    \mathcal{S}_q = \frac{\sigma^2 / G}{\sigma^2 / G + \tau^2}
\end{equation}

To ensure training stability and scalability, we maintain running estimates of $\mu_{glob}$, $\sigma^2$, and $\tau^2$ using \textbf{Welford's Online Algorithm} \citep{welford1962note}. This avoids the noise associated with small-batch statistics. For each step, the advantage for a completion $(q_m, o_{m,i}, r_{m,i})$ is computed as the deviation from this regularized baseline:
\begin{equation}
    \hat{A}_{m,i} = \frac{(r_{m,i} - V_{q_m}^{EB}) - \mu_{\hat{A}}}{\sigma_{\hat{A}} + \epsilon}
\end{equation}
where $\mu_{\hat{A}}$ and $\sigma_{\hat{A}}$ are batch-level mean and standard deviation of raw advantages. The full algorithm is presented in Algorithm~\ref{alg:EBPO} in Appendix~\ref{sec:supp_materials}.

\subsection{Theoretical Analysis of EBPO}

In this section, we analyze the theoretical properties of the EBPO estimator and formally demonstrate its advantages over the standard GRPO baseline. Proofs of results in this section can be found in Appendix \ref{sec:proof_details}.

\subsubsection{Definitions and Setup}
Let $q$ be a prompt and $\{o_1, \dots, o_G\}$ be a group of $G$ responses generated by policy $\pi_\theta$. Let $r_i \in \{0, 1\}$ be the binary reward for response $o_i$.

\textbf{GRPO Baseline:} The GRPO baseline is the local sample mean:
\begin{equation}
    V^{\text{GRPO}} = \mu_{\text{group}} = \frac{1}{G}\sum_{i=1}^G r_i
\end{equation}
The GRPO raw advantage (without normalization) is $\hat{A}^{raw}_{\text{GRPO}}(o_i) = r_i - V^{\text{GRPO}}$.

\textbf{EBPO Baseline:} The EBPO baseline is the shrinkage estimator:
\begin{equation}
    V^{\text{EBPO}} = (1 - \mathcal{S})\mu_{\text{group}} + \mathcal{S}\mu_{\text{glob}}
\end{equation}
where $\mathcal{S} \in (0, 1]$ is the shrinkage factor and $\mu_{\text{glob}} > 0$ is the global policy success rate. The EBPO raw advantage is $\hat{A}^{raw}_{\text{EBPO}}(o_i) = r_i - V^{\text{EBPO}}$.

\subsubsection{Theoretical Properties of EBPO} 
The first two theorems are straightforward from our construction. Unlike GRPO, which is entirely local, EBPO allows the model to differentiate between failing a difficult task (where the reward remains close to the baseline) and failing an easy task (where the baseline is high, leading to a large negative advantage).
\begin{theorem}[Non-Vanishing Gradients in Saturation Regimes]
\label{thm:non_vanishing}
Consider a \textbf{saturated failure group} where the policy fails on all sampled responses for a specific prompt $q$, i.e., $r_i = 0$ for all $i \in \{1, \dots, G\}$. Under these conditions:
\begin{enumerate}
    \item The GRPO gradient contribution is zero (vanishes).
    \item The EBPO gradient contribution is non-zero (informative), provided $\mu_{\text{glob}} > 0$ and the batch is not homogeneous in the sense that $\sigma_{\hat{A}} > 0$.
\end{enumerate}
\end{theorem}

\begin{theorem}[Stability via MSE Reduction]
\label{thm:mse_stability}
Assume the true latent success rate for prompt $q$ is $\theta_q$, and the observed group mean $\mu_{\text{group}}$ is an unbiased estimator of $\theta_q$ with variance $\sigma^2/G$. Under the Gaussian approximation where $\theta_q \sim \mathcal{N}(\mu_{\text{glob}}, \tau^2)$, the EBPO baseline $V^{\text{EBPO}}$ achieves strictly lower Mean Squared Error (MSE) than the GRPO baseline $\mu_{\text{group}}$ for estimating the true difficulty $\theta_q$.
\end{theorem}
\begin{remark} Unlike the unbiased sample mean in GRPO, EBPO introduces a bias towards the global prior $\mu_{glob}$ to minimize the MSE, a standard trade-off in shrinkage estimation. The global statistics $\mu_{glob}$ and $\tau^2$ are updated online using the current batch (Algorithm~\ref{alg:EBPO}), with the bias introduced by this correlation decays at a rate of $O(1/T)$, where $T$ is the total number of accumulated training steps. So for large $T$, the global statistics act as fixed priors relative to the local group updates, keeping the policy gradient asymptotically consistent.
\end{remark}
\begin{corollary}[Global Context Sensitivity]
\label{thm:global_context}
The penalty assigned by EBPO for a failure ($r_i=0$) in a saturated group scales dynamically with the global task difficulty $\mu_{\text{glob}}$. Specifically, failures on globally ``easy" tasks (high $\mu_{\text{glob}}$) incur larger penalties than failures on globally ``hard" tasks (low $\mu_{\text{glob}}$) (as visualized in Figure \ref{fig:penalty_scaling}).
\end{corollary}

\begin{figure}[bt]
    \centering
    \includegraphics[width=0.8\linewidth]{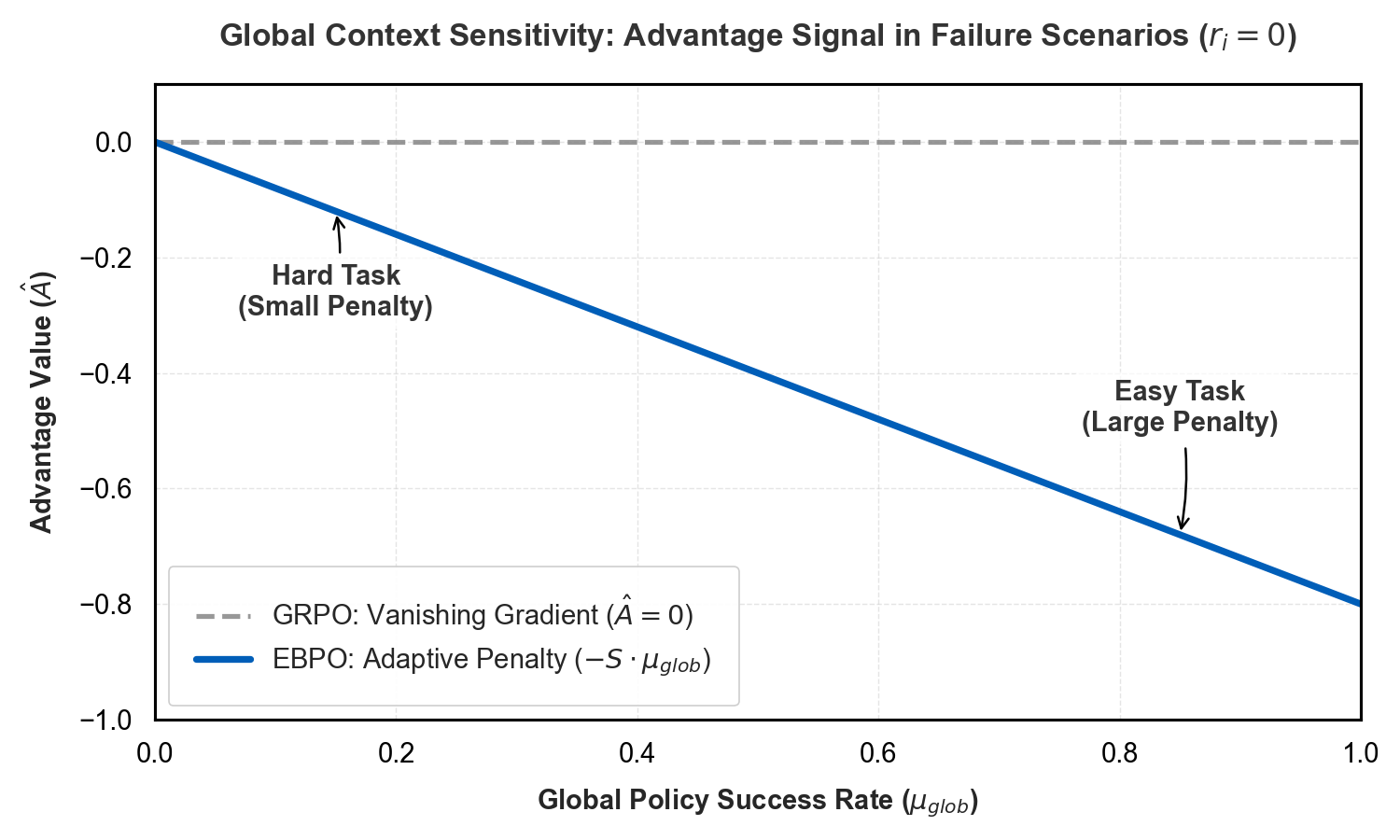}
    \caption{\textbf{Advantage Signal in Failure Scenarios ($r_i=0$ for all $i$).} 
    Comparison of the advantage signal of GRPO (dashed line) versus EBPO (solid blue line) when the model fails all attempts for a given prompt. 
    While GRPO yields a vanishing gradient ($\hat{A}=0$) regardless of task difficulty, EBPO provides a dynamic penalty signal $-\mathcal{S}\mu_{\text{glob}}$ that scales with the global success rate ($\mu_{\text{glob}}$).}
    \label{fig:penalty_scaling}
\end{figure}



We next analyze the exploration dynamics. Recent work by \citet{cui2025entropy} establishes that entropy collapse in reasoning models is driven by the covariance between policy likelihood and advantage. We posit that EBPO implicitly regularizes this decay. We note that our derivation relies on the approximation framework from \citet{cui2025entropy}.

\begin{proposition}[Entropy Conservation via Shrinkage-Regularized Covariance]
\label{thm:entropy_conservation}
Let $\Delta H(\pi) = H(\pi_t) - H(\pi_{t+1})$ denote the reduction in policy entropy after a single gradient descent step with learning rate $\eta$. 
Under the assumption that the magnitude of the likelihood-reward covariance $\Omega_q$ is independent of the local group scale $\sigma_q$, the entropy reduction of EBPO is strictly bounded by that of GRPO:
\begin{equation*}
    \mathbb{E}[\Delta H(\pi)_{\text{EBPO}}] < \mathbb{E}[\Delta H(\pi)_{\text{GRPO}}]
\end{equation*}
provided the task distribution exhibits non-zero between-group variance (i.e., $\text{Var}_q(\mu_q) > 0$), ensuring that the global normalization constant effectively dampens updates for low-variance groups.
\end{proposition}

\begin{remark}
    While the independence assumption may not strictly hold in all regimes (e.g., highly saturated groups might exhibit lower covariance), it captures the dominant dynamics of entropy decay. We empirically validate its practical hold in Section~\ref{sec:results} (Figure~\ref{fig:entropy_curve}).
\end{remark}

\subsubsection{Optimizing Prior Estimation via Clustered Sampling}
\label{sec:clustering}
While EBPO stabilizes the baseline by leveraging global history, a potential limitation arises from the heterogeneity of the training data. In a standard random shuffle, the online estimator for $\mu_{\text{glob}}$ aggregates performance across vastly different domains (e.g., mixing simple arithmetic with complex calculus) or different difficulties (e.g., mixing IMO problems with AMC 8 problems). This high variance in the data stream can cause the global prior to converge to a coarse average that creates a ``mismatch" for specific, distinct tasks.

To address this, we propose organizing the training stream into coherent clusters. We explore two strategies:
\begin{enumerate}
    \item \textbf{Topic Clustering:} Grouping prompts by mathematical domain (e.g., Algebra, Geometry).
    \item \textbf{Difficulty Clustering (Curriculum):} Ordering prompts by success rate from easy to hard.
\end{enumerate}
By presenting data in clustered sequences, the online estimator adapts to consistent local distributions rather than oscillating between extremes. We formally justify this approach with the following proposition, which shows that clustered sampling reduces the estimation error of the prior.

\begin{proposition}[Advantage of Clustered Sampling]
\label{prop:clustering}
Let the dataset $\mathcal{D}$ be partitioned into $K$ distinct clusters (topics) $\{\mathcal{C}_1, \dots, \mathcal{C}_K\}$. Let $\theta_q$ be the latent success rate of a prompt $q$. Assume the true difficulty varies by topic, such that $\mu_k = \mathbb{E}[\theta_q \mid q \in \mathcal{C}_k]$ varies across $k$, while the global mean is $\mu_{\text{glob}} = \mathbb{E}_k[\mu_k]$.

Let $\hat{\mu}_{\text{prior}}$ be the global statistic used by EBPO. We define two streaming regimes:
\begin{itemize}
    \item \textbf{Random Shuffle:} The stream is sampled uniformly from $\mathcal{D}$, such that $\hat{\mu}_{\text{prior}}$ converges to $\mu_{\text{glob}}$.
    \item \textbf{Topic-Coherent:} The stream is ordered by cluster. For a prompt $q \in \mathcal{C}_k$, the estimator has adapted such that $\hat{\mu}_{\text{prior}} = \mu_k$.
\end{itemize}
The Mean Squared Error (MSE) of the prior estimate w.r.t the true difficulty $\theta_q$ is strictly minimized in the Topic-Coherent regime.
\end{proposition}
\section{Experiments}
\label{sec:experiments}

We first outline the experimental setup. Then, we compare the proposed \ours method with GRPO and its variants under different experimental settings.

\subsection{Experimental Setup}
\label{sec:setup}

\paragraph{Dataset.}
We use the DAPO-Math-17K dataset as the training corpus \citep{yu2025dapo}. For evaluation, we select a suite of competitive math reasoning benchmarks, including AIME2024 \citep{li2024numinamath}, AIME2025 \citep{li2024numinamath}, AMC23 \citep{li2024numinamath}, Math-500 \citep{hendrycks2021measuring}, and OlympiadBench \citep{he2024olympiadbench}. Each model is evaluated using the Pass@1 metric. To ensure statistical robustness, we repeat each evaluation set 32 times with different random seeds and report the average Pass@1 across runs.

\paragraph{Models.} 
We adopt a diverse set of LLMs with different architectures and parameter scales. Specifically, we use \textbf{LLaMA3.1-8B} \citep{dubey2024llama}, \textbf{Qwen3-8B}, and \textbf{Qwen3-14B} \citep{yang2025qwen3}. This selection spans multiple LLM families (LLaMA and Qwen) and model sizes, providing a robust and comprehensive testbed for evaluating our agent workflow and training recipe.

\paragraph{Baseline.}
We compare our proposed \textbf{\ours} method with several representative baselines, including Naive GRPO \citep{shao2024deepseekmath}, DAPO \citep{yu2025dapo}, Dr. GRPO \citep{liu2025understanding}, and EntropyMech \citep{cui2025entropy}. 
For a fair comparison, all methods are trained under identical conditions: the same training data ordering, batch size, and optimization configuration are applied across all baselines and our \ours method. Specific configuration details for each baseline are provided in Appendix~\ref{sec:hyper_details}.

\paragraph{Clustered Sampling Details.} 
To implement the clustered sampling optimization in Section \ref{sec:clustering}, we evaluate two sampling strategies: \textit{clustering by topic} and \textit{clustering by difficulty}, which correspond to \textbf{\ours-topic} and \textbf{\ours-diff}, respectively. 
For difficulty-based clustering, we use the base model's pass rate as the difficulty metric, grouping problems of similar success rates together. 
For topic-based clustering, we employ GPT-4.1 \citep{achiam2023gpt} to annotate each problem with one topic label. 
Details are provided in Appendix~\ref{sec:hyper_details}.

\subsection{Results} \label{sec:results}

\begin{table*}[t]
\centering
\begin{tabular}{lccccc|c}
\toprule
\textbf{Method} & \textbf{MATH-500} & \textbf{AIME-2024} & \textbf{AIME-2025} & \textbf{AMC23} & \textbf{OlympiadBench} & \textbf{Average} \\
\midrule
\multicolumn{7}{l}{\textbf{Qwen3-8B}} \\
\midrule
\ours-topic & \textbf{76.80} & \textbf{56.04} & \textbf{47.92} & 86.25 & \textbf{54.93} & \textbf{64.39} \\
GRPO & 65.60 & 50.21 & 42.29 & \textbf{89.53} & 45.99 & 58.72 \\
Dr~GRPO & 67.68 & 51.04 & 32.71 & 85.00 & 44.91 & 56.67 \\
DAPO & 58.39 & 45.63 & 32.71 & 82.81 & 43.62 & 52.63 \\
EntropyMech & 53.88 & 37.92 & 30.42 & 79.99 & 43.69 & 49.18 \\
\midrule
\multicolumn{7}{l}{\textbf{LLaMA3-8B}} \\
\midrule
\ours-topic & \textbf{23.84} & 2.29 & 0.21 & \textbf{15.78} & \textbf{7.75} & \textbf{9.97} \\
GRPO & 22.66 & 2.71 & 0.00 & 12.66 & 6.97 & 9.00 \\
Dr~GRPO & 22.12 & \textbf{2.91} & \textbf{0.42} & 12.36 & 6.65 & 8.89 \\
DAPO & 18.05 & 2.29 & 0.00 & 11.09 & 6.05 & 7.90 \\
EntropyMech & 15.01 & 1.25 & 0.21 & 10.16 & 5.30 & 6.79 \\
\midrule
\multicolumn{7}{l}{\textbf{Qwen3-14B}} \\
\midrule
\ours-topic & 71.01 & \textbf{58.13} & 45.63 & \textbf{85.47} & 48.52 & \textbf{61.35} \\
GRPO & \textbf{75.73} & 49.99 & 37.92 & 85.00 & 50.77 & 59.88 \\
Dr~GRPO & 73.63 & 53.33 & 39.79 & 82.97 & \textbf{51.08} & 60.16 \\
DAPO & 67.68 & 57.29 & \textbf{51.46} & 84.22 & 46.03 & 61.34 \\
EntropyMech & 59.94 & 49.37 & 40.83 & 78.75 & 41.21 & 54.42 \\
\bottomrule
\end{tabular}
\caption{
Performance comparison of \ours and baseline methods across different evaluation datasets (group size = 4). 
The dataset is clustered by topics during training. Each value represents Pass@1 (\%). 
The best result within a base model group is highlighted in bold.
}
\label{tab:ebpo_results}
\end{table*}

\begin{figure*}[t]
    \centering
    \begin{subfigure}[t]{0.45\textwidth}
        \centering
        \includegraphics[width=\textwidth]{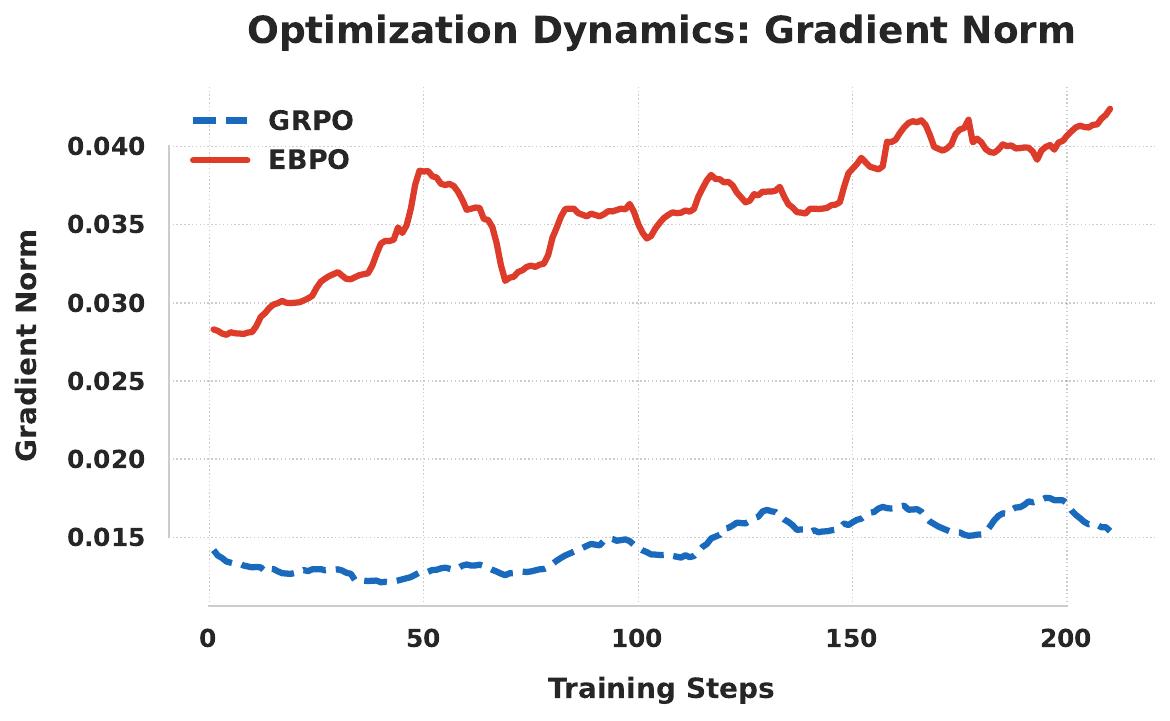}
        \caption{Evolution of Gradient Norm ($||\nabla_\theta J||_2$).}
        \label{fig:grad_norm}
    \end{subfigure}
    \hfill
    \begin{subfigure}[t]{0.45\textwidth}
        \centering
        \includegraphics[width=\textwidth]{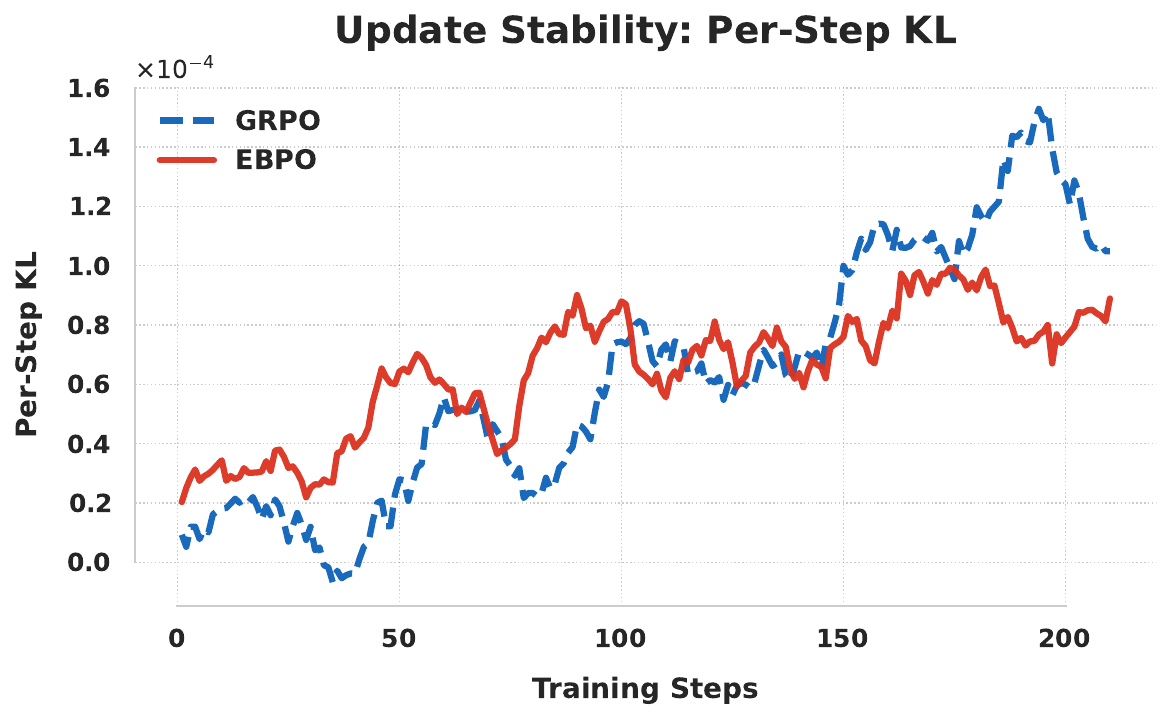}
        \caption{Evolution of Per-Step KL ($\text{KL}(\pi_t || \pi_{t+1})$).}
        \label{fig:per_step_kl}
    \end{subfigure}
    \caption{
    \textbf{Optimization stability analysis.} \textbf{(a)} \ours maintains healthy, non-vanishing gradients compared to GRPO's diminishing signals. \textbf{(b)} \ours prevents policy collapse by strictly bounding the update magnitude (per-step KL) throughout training.
    }
    \label{fig:train_dynamics}
\end{figure*}

\paragraph{\ours consistently outperforms GRPO and its variants across multiple reasoning benchmarks and model sizes.}

We report results obtained under the setting where the training data are clustered by topic and the group size of GRPO is 4, as summarized in Table~\ref{tab:ebpo_results}. 
Across all evaluated models and benchmarks, \textbf{\ours-topic} achieves clear and consistent improvements over prior value-free reinforcement learning baselines. 
For example, on the Qwen3-8B model, \ours reaches an average Pass@1 of \textbf{64.39\%}, outperforming GRPO by more than \textbf{5\%}. 
Similar trends are observed on other models, indicating that the benefits of \ours generalize across architectures and scales. Overall, among the 15 possible (model, dataset) combinations, \ours achieves the best result in 9 cases, demonstrating its robustness across both architectures and reasoning tasks.

These results highlight the advantage of \ours, which stabilizes gradient estimation across prompt groups and improves reward signal utilization. 
As a result, \ours produces stronger and more reliable reasoning performance compared to existing GRPO-based approaches.

\textbf{\ours achieves stable and non-vanishing policy optimization.} To analyze the training dynamics, we examine the magnitude of policy updates and gradient flow for the Qwen3-8B training. First, in Figure~\ref{fig:grad_norm}, we plot the gradient norm ($||\nabla_\theta J||_2$). Standard GRPO exhibits consistently lower gradient magnitudes, indicative of \textbf{vanishing gradients} in saturated regimes where all group responses fail and yield zero relative advantage. In contrast, EBPO maintains a robust gradient flow throughout training, confirming that the shrinkage-regularized baseline provides informative penalty signals even when local variance is zero. Second, we track the per-step KL divergence ($\text{KL}(\pi_t || \pi_{t+1})$) in Figure~\ref{fig:per_step_kl}. While GRPO shows signs of instability in later stages, with update sizes spiking unpredictably, EBPO maintains a strictly bounded update magnitude. These metrics demonstrate that EBPO strikes a critical balance: it prevents the policy from stalling due to vanishing signals while regularizing it against the catastrophic large-step deviations that lead to model collapse.


\textbf{\ours demonstrates stronger exploration behavior compared to GRPO.}
We explicitly verify the theoretical claims of Proposition~\ref{thm:entropy_conservation} by tracking the policy entropy in Figure~\ref{fig:entropy_curve}. \ours consistently maintains \textbf{higher entropy} during training compared to GRPO, suggesting that it encourages greater exploration and prevents premature policy collapse. This aligns with the algorithmic design of \ours: by incorporating a batch-level shrinkage baseline, the policy updates are better balanced between exploitation of high-reward samples and exploration of underrepresented regions in the solution space. This preservation of exploration capability has a direct impact on generalization performance. As shown in Figure~\ref{fig:val_performance}, we track the evolution of validation reward (Majority Vote@16 and Pass@16) throughout training. While GRPO plateaus due to mode collapse, EBPO maintains a continuous upward trajectory in both Majority Vote@16 and Pass@16, yielding a policy with higher robustness and peak capability. Furthermore, We also conduct an ablation study in Appendix~\ref{appendix:ablation_clustering} to investigate the impact of topic clustering on the efficacy of the \ours framework.

\begin{figure}[!htb]
\centering
    \includegraphics[width=0.5\textwidth]{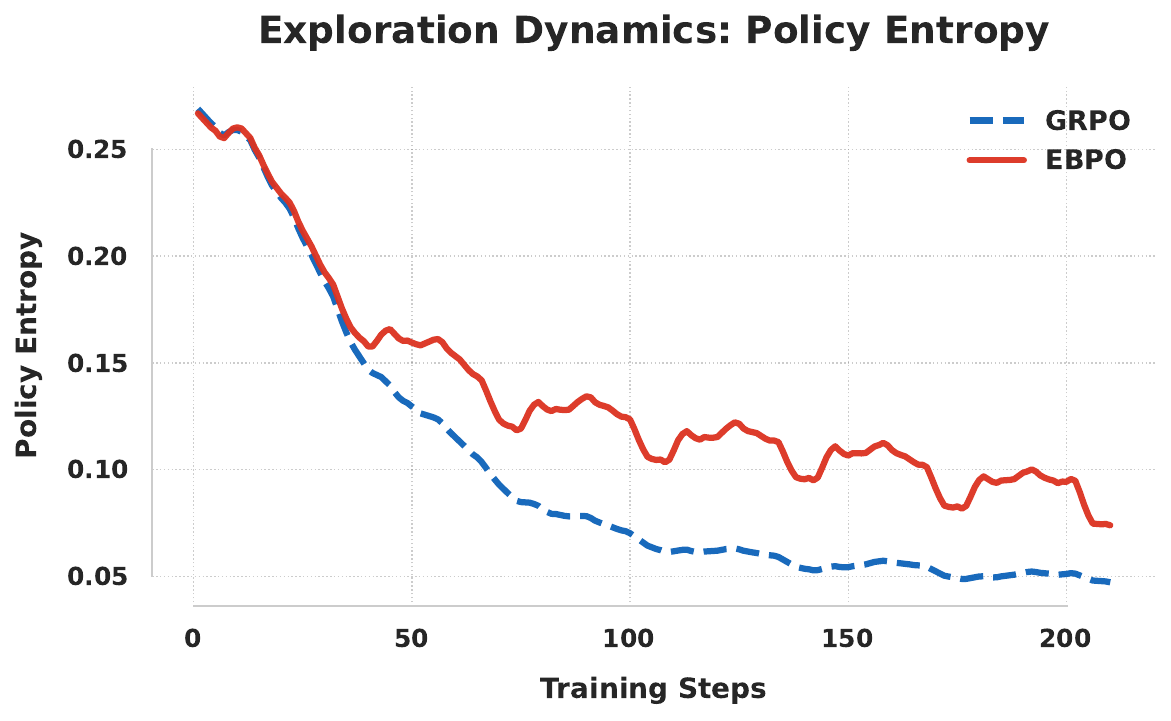}
    \caption{
    \textbf{Evolution of Policy Entropy ($G=4$).} \ours maintains a consistently higher policy entropy than GRPO, demonstrating that the global prior effectively sustains exploration throughout training.}
    \label{fig:entropy_curve}
\end{figure}

\begin{figure*}[t]
    \centering
    \begin{subfigure}[t]{0.45\textwidth}
        \centering
        \includegraphics[width=\textwidth]{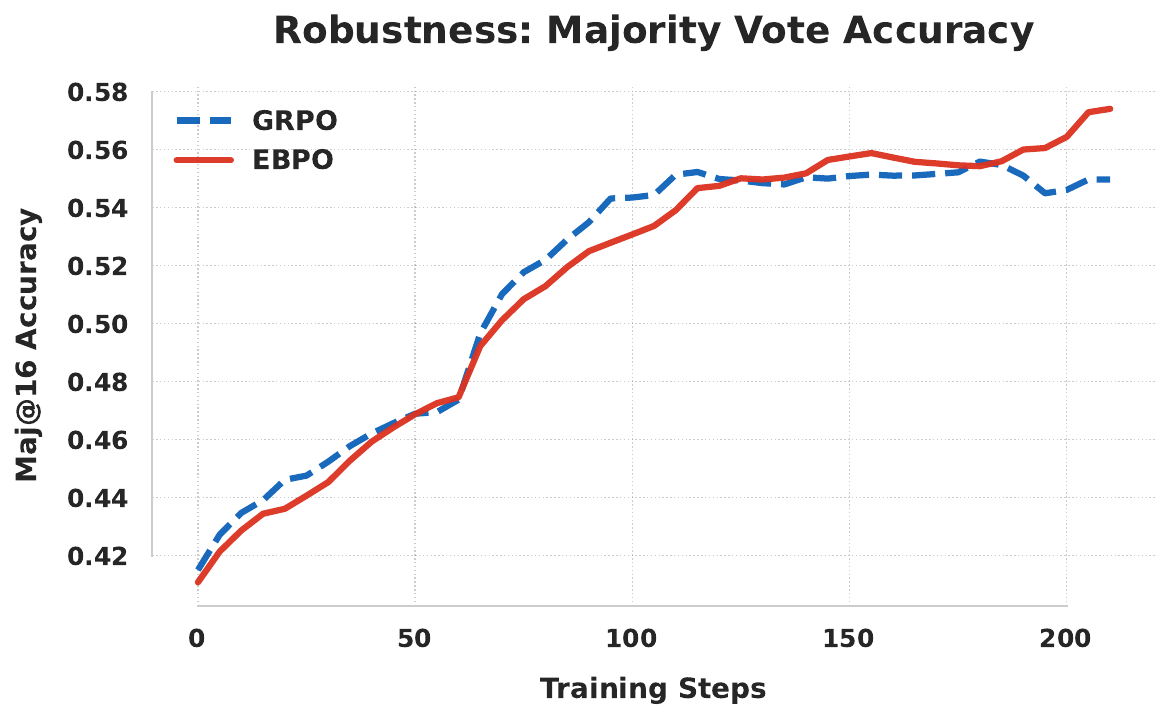}
        \label{fig:maj_vote}
    \end{subfigure}
    \hfill
    \begin{subfigure}[t]{0.45\textwidth}
        \centering
        \includegraphics[width=\textwidth]{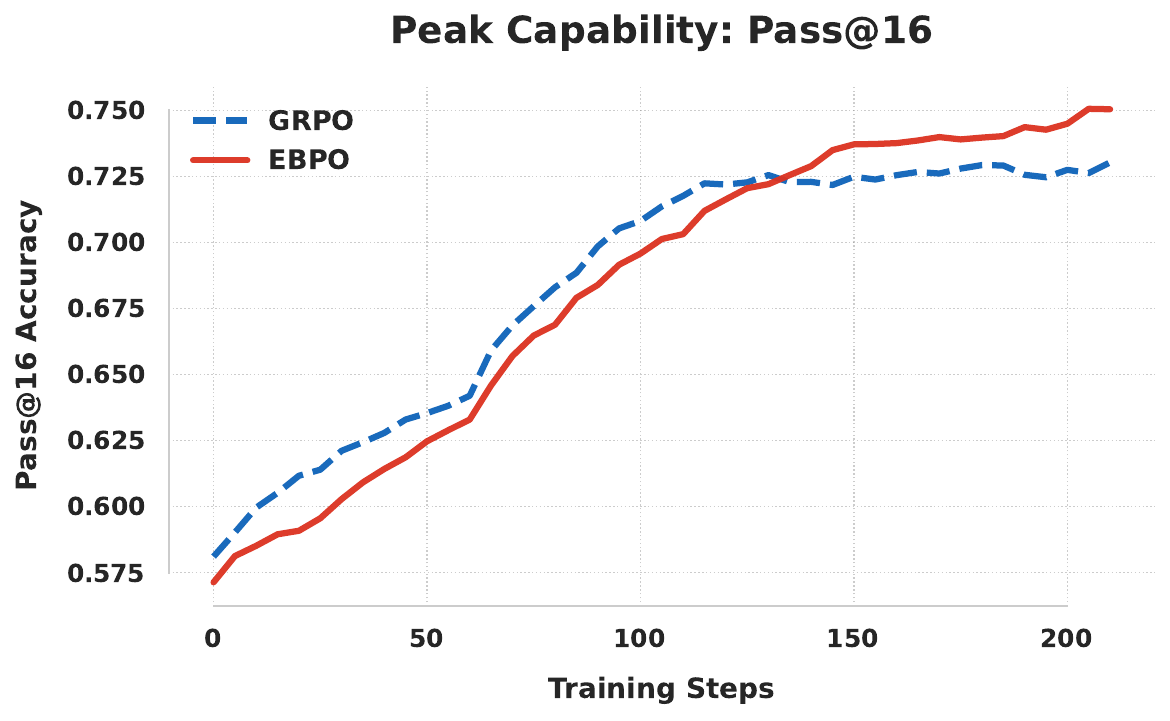}
        \label{fig:pass_16}
    \end{subfigure}
    \vspace{-1em}
    \caption{
    \textbf{Validation Performance.} While GRPO plateaus or degrades in late-stage training—indicative of overfitting or policy collapse—EBPO exhibits sustained performance gains.
    }
    \label{fig:val_performance}
\end{figure*}

In summary, EBPO stabilizes optimization by mitigating vanishing gradients and strictly bounding update magnitudes, effectively preventing both stagnation and model collapse. This algorithmic resilience, combined with sustained exploration via entropy conservation, drives the consistent performance gains observed across reasoning benchmarks.

\begin{table}[ht]
\centering
\begin{tabular}{lccccccc}
\toprule
\textbf{Method} & $G$ & \textbf{MATH} & \textbf{AIME24} & \textbf{AIME25} & \textbf{AMC23} & \textbf{Olympiad} & \textbf{Avg.} \\
\midrule
\ours-topic      & 8  & \textbf{74.26} & \textbf{64.37} & \textbf{48.96} & \textbf{89.69} & \textbf{50.44} & \textbf{65.54} \\
GRPO       & 8  & 68.79 & 47.29 & 31.46 & 79.69 & 44.06 & 54.26 \\
DAPO       & 8  & 56.63 & 35.42 & 28.33 & 77.34 & 35.75 & 46.69 \\
\midrule
\ours-topic      & 16 & 77.39 & \textbf{60.00} & \textbf{46.46} & \textbf{85.94} & \textbf{53.41} & \textbf{64.64} \\
GRPO       & 16 & \textbf{77.98} & 55.42 & 44.58 & 85.94 & 51.89 & 63.16 \\
DAPO       & 16 & 73.26 & 46.67 & 37.92 & 76.87 & 47.07 & 56.36 \\
\midrule
\ours-topic      & 32 & 72.24 & \textbf{61.04} & \textbf{44.79} & 85.47 & 48.52 & 62.41 \\
GRPO       & 32 & \textbf{76.94} & 59.37 & 40.00 & \textbf{86.88} & \textbf{51.34} & \textbf{62.91} \\
DAPO       & 32 & 72.71 & 46.67 & 35.21 & 83.13 & 46.77 & 56.90 \\
\bottomrule
\end{tabular}
\caption{Performance comparison across different group sizes. \ours consistently demonstrates superior sample efficiency.}
\label{tab:group_size_results}
\end{table}

\subsection{Sensitivity to Group Size}
\textbf{\ours consistently achieves superior sample efficiency, reaching high-tier performance at small group sizes.} The performance of group-based reinforcement learning methods is fundamentally tied to the group size $G$, which governs the fidelity of the advantage signal. Theoretically, standard GRPO computes advantage by normalizing rewards against the empirical mean and standard deviation of the group. In regimes where $G$ is small, this estimation is prone to high variance based on noise. \ours addresses this by grounding the advantage in topic-level evidence, which provides a more stable semantic anchor for policy updates even when few responses are sampled per prompt. To evaluate this stability, we compare \ours, GRPO, and DAPO across $G \in \{8, 16, 32\}$ on the Qwen3-8B model. The results are detailed in Table~\ref{tab:group_size_results}.

The scaling behavior across $G=4$ to $G=32$ reveals a critical trend: \ours is remarkably sample-efficient. While standard GRPO and DAPO show significant performance gaps at smaller group sizes (e.g., $G=8$), \ours achieves high-tier performance even with limited samples. Notably, at $G=8$, \ours outperforms GRPO by 11.28\% on average, suggesting that topic-based evidence gathering effectively ``densifies'' the sparse rewards found in smaller groups. 

As $G$ increases, the performance of the baselines improves as their empirical statistics become more reliable. For \ours, the most substantial gains are realized at lower $G$ values, effectively reaching a performance plateau earlier than the baselines. This indicates that \ours is well-suited for training environments with computational constraints, providing a robust training signal where traditional group-based normalization struggles with high variance.

\subsection{Curriculum Learning with Difficulty Clustering}
\textbf{Curriculum learning provides a more stable training signal for \ours, consistently outperforming GRPO.}

To evaluate the impact of data ordering, we train various models for one epoch on a dataset re-ranked from easy to hard difficulty with a group size $G=4$. This structured exposure, referred to as \ours-diff, aligns the complexity of the topic-level evidence clusters with the model's evolving capability. By starting with simpler samples, the model establishes a reliable logical foundation before being exposed to highly intricate reasoning chains. The comparative results are presented in Figure \ref{fig:curriculum_results}.


\begin{figure}[ht]
    \centering
    \includegraphics[width=0.8\linewidth]{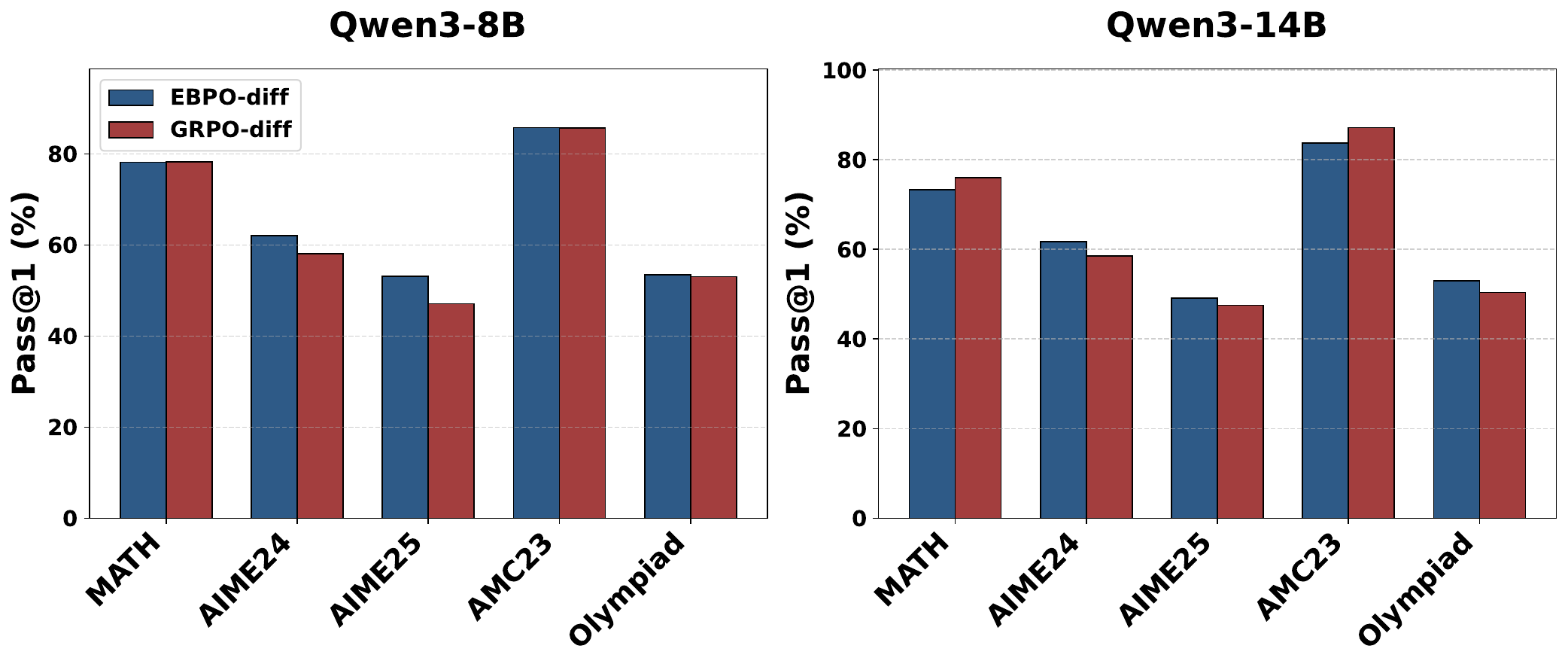}
    \caption{Performance comparison under difficulty-based curriculum learning.}
    \label{fig:curriculum_results}
\end{figure}



As illustrated in Figure~\ref{fig:curriculum_results}, \ours-diff shows a clear performance advantage over the GRPO baseline on high-difficulty benchmarks. Specifically, for Qwen3-8B, \ours-diff outperforms GRPO-diff by 3.95\% on AIME24 and 6.04\% on AIME25; a consistent lead is also observed for Qwen3-14B on AIME and OlympiadBench. This performance gap suggests that while standard group-based normalization suffices for conventional tasks like MATH and AMC23, grounding advantage estimation in stabilized evidence clusters is significantly more effective for the sparse reward landscapes of elite-level competition. By using a difficulty-based curriculum to isolate strong logical rationales early in training, \ours-diff reduces evidence noise and enables more effective transfer to complex, multi-step reasoning needed for competition-level problems.

\section{Conclusion}
In this paper, we presented \ours, a framework that enhances policy optimization by utilizing an Empirical Bayes approach to estimate reward mean and variance across clustered logical evidence. By replacing standard group-wide normalization with evidence-grounded statistics, \ours significantly improves sample efficiency. Furthermore, our difficulty-based curriculum stabilizes the estimation of these empirical priors, leading to substantial performance gains on reasoning benchmarks. Ultimately, \ours demonstrates that grounding reinforcement learning in shared semantic rationales provides a theoretically robust and computationally efficient path to elite-level mathematical reasoning.


\clearpage
\newpage
\bibliographystyle{assets/plainnat}
\bibliography{example_paper}

\clearpage
\newpage
\beginappendix

\section{Hyperparameter Details}
\label{sec:hyper_details}
We implement the \ours framework and all baselines using the \texttt{verl} package. For optimization, we utilize a mini-batch size of 128 per backpropagation step, yielding a total global batch size of 512. The learning rate is fixed at $1 \times 10^{-6}$, with a KL-divergence coefficient ($\beta$) set to $0.001$. For the specific hyperparameters associated with the DAPO baseline, we adhere strictly to the configurations established in the original study \citep{yu2025dapo}.

To evaluate the impact of difficulty-based curriculum learning, we train for a single epoch without data shuffling to maintain the easy-to-hard ordering, utilizing the final checkpoint for performance assessment. For topic-level clustering experiments, training is conducted for up to 1,000 steps. In this setting, we select the optimal checkpoint based on the highest performance achieved on a held-out validation set, evaluated using a majority vote over 16 rollouts. All baseline comparisons are conducted under identical hyperparameter configurations to ensure a fair evaluation.

To quantify the difficulty of each task relative to the base model, we perform sixteen independent rollouts per prompt using a sampling temperature of $T=1.0$. The difficulty metric for each question is defined as the empirical pass rate, calculated as the percentage of correct trajectories generated by the base model. 

For semantic topic categorization, we establish nine fixed domains within the mathematical dataset: \textit{Algebra \& Number Theory}, \textit{Calculus \& Analysis}, \textit{Geometry \& Trigonometry}, \textit{Discrete Mathematics}, \textit{Probability \& Statistics}, \textit{Linear Algebra}, \textit{Applied Mathematics \& Word Problems}, \textit{Physics \& Chemistry}, and \textit{Advanced Mathematics}. We leverage GPT-4.1 \citep{achiam2023gpt} to classify each question into exactly one of these categories. This high-level semantic labeling facilitates the topic-level evidence aggregation required for the Empirical Bayes advantage estimation in the \ours framework.

\section{Supplementary Materials}
\label{sec:supp_materials}

\subsection{Ablation Study: The Impact of Topic Clustering}
\label{appendix:ablation_clustering}

To isolate the contribution of semantic topic clustering to the overall performance of \ours, we compare the proposed \ours-topic against a baseline variant, \ours-naive. In the \ours-naive setting, the Empirical Bayes shrinkage mechanism is applied to standard, randomly shuffled training batches where prompts are likely to be semantically heterogeneous. This comparison allows us to verify if ``borrowing strength'' is a general statistical benefit or if its efficacy is tied to the semantic consistency of the evidence clusters.

\begin{table}[ht]
\centering
\small
\begin{tabular}{lccccc}
    \toprule
    \textbf{Method} & \textbf{MATH-500} & \textbf{AIME24} & \textbf{AIME25} & \textbf{AMC23} & \textbf{OlympiadBench} \\
    \midrule
    \multicolumn{6}{l}{\textbf{Qwen3-8B}} \\
    \ours-topic & \textbf{76.80} & \textbf{56.04} & \textbf{47.92} & 86.25 & \textbf{54.93} \\
    \ours-naive & 68.58 & 53.75 & 41.88 & \textbf{87.34} & 48.44 \\
    \midrule
    \multicolumn{6}{l}{\textbf{Qwen3-14B}} \\
    \ours-topic & 71.01 & \textbf{58.13} & \textbf{45.63} & 85.47 & \textbf{48.52} \\
    \ours-naive & \textbf{72.08} & 52.29 & 40.21 & \textbf{86.56} & 48.41 \\
    
    \bottomrule
    \end{tabular}
\caption{Ablation results comparing topic-clustered Empirical Bayes (\ours-topic) against naive, shuffled-batch Empirical Bayes (\ours-naive). Results are reported as Pass@1 (\%).}
\label{tab:ablation_clustering}
\end{table}

The results in Table~\ref{tab:ablation_clustering} demonstrate that semantic clustering is a vital prerequisite for effective Empirical Bayes regularization. Across both model scales, \ours-topic consistently outperforms the naive version on high-difficulty reasoning benchmarks like AIME and OlympiadBench. For instance, in the Qwen3-8B model, topic clustering improves the AIME25 score by 6.04\%. 

The failure of \ours-naive to reach similar heights can be attributed to the nature of the global prior. When a batch contains unrelated topics, the estimated prior mean $\mu_{\text{glob}}$ and variance $\tau^2$ represent a broad, high-entropy distribution of rewards that may not accurately describe the latent reward potential of any specific prompt. Consequently, the shrinkage estimator $V_q^{EB}$ pulls the local mean $\mu_{\text{group}}$  toward a noisy center, potentially suppressing valid gradient signals or introducing excessive bias. By contrast, clustering by topic ensures that the prior $\mu_{\text{glob}}$ is derived from semantically related tasks, providing a much sharper and more relevant baseline for advantage regularization.

\begin{algorithm}[tb]
   \caption{EBPO Advantage Estimation}
   \label{alg:EBPO}
\begin{algorithmic}[1]
   \STATE {\bfseries Input:} Policy $\pi_\theta$, Reference policy $\pi_{\text{ref}}$, Group size $G$
   \STATE Initialize Welford statistics $\mathcal{W}_{\text{glob}}$ and $\mathcal{W}_{\text{means}}$
   \FOR{each training iteration}
       \STATE Sample batch of prompts $\{q_1, \dots, q_M\} \sim \mathcal{D}$
       \FOR{$m=1$ {\bfseries to} $M$}
           \STATE Sample $G$ responses $\{o_{m,1}, \dots, o_{m,G}\} \sim \pi_\theta(\cdot | q_m)$
           \STATE Compute rewards $r_{m,i} \in \{0, 1\}$ and $\mu_{\text{group}, m} = \frac{1}{G}\sum_{i=1}^G r_{m,i}$
           \STATE Update $\mathcal{W}_{\text{glob}}$ with all $\{r_{m,i}\}_{i=1}^G$
           \STATE Update $\mathcal{W}_{\text{means}}$ with $\mu_{\text{group}, m}$
       \ENDFOR
       
       \STATE $\mu_{\text{glob}} \gets \text{mean}(\mathcal{W}_{\text{glob}})$
       \STATE $\sigma^2 \gets \text{variance}(\mathcal{W}_{\text{glob}})$
       \STATE $\tau^2 \gets \text{variance}(\mathcal{W}_{\text{means}})$

       \FOR{$m=1$ {\bfseries to} $M$}
    \STATE $\mathcal{S}_{q_m} \gets \frac{\sigma^2 / G}{\sigma^2 / G + \tau^2}$ \COMMENT{Compute shrinkage factor}
    \STATE $V_{q_m}^{EB} \gets (1 - \mathcal{S}_{q_m}) \mu_{\text{group}, m} + \mathcal{S}_{q_m} \mu_{\text{glob}}$ \COMMENT{EBPO Baseline}
    \FOR{$i=1$ {\bfseries to} $G$}
        \STATE $\hat{A}_{m,i}^{\text{raw}} \gets r_{m,i} - V_{q_m}^{EB}$ \COMMENT{Initial advantage}
    \ENDFOR
\ENDFOR

\STATE $\mu_{\hat{A}} \gets \text{mean}(\{\hat{A}_{m,i}^{\text{raw}}\})$
\STATE $\sigma_{\hat{A}} \gets \text{std}(\{\hat{A}_{m,i}^{\text{raw}}\})$

\FOR{$m=1$ {\bfseries to} $M$}
    \FOR{$i=1$ {\bfseries to} $G$}
        \STATE $\hat{A}_{m,i} \gets \frac{\hat{A}_{m,i}^{\text{raw}} - \mu_{\hat{A}}}{\sigma_{\hat{A}} + \epsilon}$ \COMMENT{Final batch-normalized advantage}
    \ENDFOR
\ENDFOR
       
       \STATE $\theta \gets \text{Update via Clipped Surrogate Objective with } \hat{A}_{m,i}$
   \ENDFOR
\end{algorithmic}
\end{algorithm}

\section{Proof Details}
\label{sec:proof_details}

\begin{proof}[Proof of Theorem~\ref{thm:non_vanishing}]

\textbf{1. GRPO Case:}
If $r_i = 0$ for all $i$, then the group mean is $\mu_{\text{group}} = \frac{1}{G}\sum 0 = 0$.
The raw advantage for any response $o_i$ is:
$$ \hat{A}^{raw}_{\text{GRPO}}(o_i) = r_i - \mu_{\text{group}} = 0 - 0 = 0 $$
Consequently, the gradient update $\nabla J = \mathbb{E}[\hat{A} \nabla \log \pi] = 0$. The model receives no learning signal.

\textbf{2. EBPO Case:}
With $\mu_{\text{group}} = 0$, the EBPO baseline becomes:
$$ V^{\text{EBPO}} = (1 - \mathcal{S})(0) + \mathcal{S}\mu_{\text{glob}} = \mathcal{S}\mu_{\text{glob}} $$
The raw advantage for any response $o_i$ is:
$$ \hat{A}^{raw}_{\text{EBPO}}(o_i) = r_i - V^{\text{EBPO}} = 0 - \mathcal{S}\mu_{\text{glob}} = -\mathcal{S}\mu_{\text{glob}} $$
Since $\mathcal{S} > 0$ and $\mu_{\text{glob}} > 0$ (assuming the policy has succeeded at least once in history), we have $\hat{A}^{raw}_{\text{EBPO}} < 0$. This yields a non-zero negative gradient, penalizing the generation of incorrect responses even when no correct response is present in the group.
\end{proof}

\begin{proof}[Proof of Theorem~\ref{thm:mse_stability}]
The Mean Squared Error of an estimator $\hat{\theta}$ is defined as $\mathbb{E}[(\hat{\theta} - \theta_q)^2]$.
For the GRPO baseline (Sample Mean):
$$ \text{MSE}(V^{\text{GRPO}}) = \text{Var}(\mu_{\text{group}}) = \frac{\sigma^2}{G} $$

For the EBPO baseline (Shrinkage Estimator), we seek to minimize the Bayes Risk (Expected MSE). The optimal linear estimator of the form $V = (1-w)\mu_{\text{group}} + w\mu_{\text{glob}}$ minimizes:
$$ \mathcal{L}(w) = \mathbb{E}_{\theta, r} [((1-w)\mu_{\text{group}} + w\mu_{\text{glob}} - \theta_q)^2] $$
Standard Bayesian derivation shows the optimal weight $w^*$ (which corresponds to our shrinkage factor $\mathcal{S}$) is:
$$ \mathcal{S} = \frac{\text{Var}(\mu_{\text{group}})}{\text{Var}(\mu_{\text{group}}) + \text{Var}(\theta_q)} = \frac{\sigma^2/G}{\sigma^2/G + \tau^2} $$
Substituting $\mathcal{S}$ back into the MSE equation yields:
$$ \text{MSE}(V^{\text{EBPO}}) = (1 - \mathcal{S}) \frac{\sigma^2}{G} $$
Since $0 < \mathcal{S} < 1$ (provided $\tau^2 > 0$), it follows that:
$$ \text{MSE}(V^{\text{EBPO}}) < \frac{\sigma^2}{G} = \text{MSE}(V^{\text{GRPO}}) $$
Thus, EBPO provides a more stable, lower-variance estimate of the prompt's difficulty, stabilizing the advantage computation when $G$ is small.
\end{proof}

\begin{proof}[Proof of Corollary~\ref{thm:global_context}]
From the result of Theorem \ref{thm:non_vanishing}, the raw advantage for a failure in a saturated group ($r_i=0, \forall i$) is:
$$ \hat{A}^{raw}_{\text{EBPO}}(fail) = -\mathcal{S}\mu_{\text{glob}} $$
Let us compare two hypothetical regimes represented by the global history:
\begin{enumerate}
    \item \textbf{Easy Regime:} The policy generally succeeds, so $\mu_{\text{glob}} \approx 1$.
    \item \textbf{Hard Regime:} The policy generally fails, so $\mu_{\text{glob}} \approx \epsilon$ (where $0 < \epsilon \ll 1$).
\end{enumerate}
Taking the magnitude of the penalty (gradient signal):
$$ |\hat{A}^{raw}_{\text{easy}}| \propto \mathcal{S}(1) \quad \text{and} \quad |\hat{A}^{raw}_{\text{hard}}| \propto \mathcal{S}(\epsilon) $$
Clearly, $|\hat{A}^{raw}_{\text{easy}}| \gg |\hat{A}^{raw}_{\text{hard}}|$. 
EBPO uses the global context to differentiate these scenarios: it applies a strong correction signal when the model fails a task that is statistically ``easy'' (deviating significantly from the prior), but applies a gentle correction when the model fails a task that is ``hard'' (consistent with the prior), thereby preventing catastrophic forgetting or instability on difficult tasks.
\end{proof}

\begin{proof}[Proof of Proposition~\ref{thm:entropy_conservation}]
We invoke the result from \citet{cui2025entropy}, which establishes that the entropy decay is dominated by the covariance between the policy likelihood and the estimated advantage:
\begin{equation*}
    \Delta H(\pi) \approx \eta \mathbb{E}_{q \sim \mathcal{D}} \left[ \sum_{o} \text{Cov}_{\pi}\left( \log \pi_\theta(o|q), \hat{A}(o, q) \right) \right]
\end{equation*}

We analyze the sensitivity of the entropy update to the advantage normalization schemes. Let $\Omega_q = \text{Cov}_{\pi}(\log \pi(o|q), r(o))$ denote the alignment between the policy's log-likelihood and the raw rewards for a specific group $q$.

In GRPO, the advantage is normalized by the local group statistics. The advantage is $\hat{A}_{\text{GRPO}}(o) = \frac{r(o) - \mu_q}{\sigma_q}$. Substituting this into the entropy update:
\begin{align*}
    \Delta H_{\text{GRPO}} &\approx \eta \mathbb{E}_{q} \left[ \text{Cov}\left( \log \pi, \frac{r - \mu_q}{\sigma_q} \right) \right] \\
    &= \eta \mathbb{E}_{q} \left[ \frac{1}{\sigma_q} \text{Cov}( \log \pi, r - \mu_q ) \right] \\
    &= \eta \mathbb{E}_{q} \left[ \frac{\Omega_q}{\sigma_q} \right] \\
    & \text{(since $\mu_q$ is constant w.r.t local covariance)}
\end{align*}
This update is sensitive to $1/\sigma_q$. For "saturated" prompts where responses are highly similar, $\sigma_q \to 0$, causing the update magnitude to diverge.

In EBPO, the raw advantage is $A^{\text{raw}} = r(o) - V^{EBPO}_q$. The final advantage is standardized over the entire batch:
\begin{equation*}
    \hat{A}_{\text{EBPO}}(o) = \frac{A^{\text{raw}} - \mu_{\hat{A}}}{\sigma_{\text{batch}}} = \frac{r(o) - V^{EBPO}_q - \mu_{\hat{A}}}{\sigma_{\text{batch}}}
\end{equation*}
where $\mu_{\hat{A}}$ and $\sigma_{\text{batch}}$ are the mean and standard deviation of the raw advantages across the entire batch.
Substituting this into the entropy update:
\begin{align*}
    \Delta H_{\text{EBPO}} &\approx \eta \mathbb{E}_{q} \left[ \text{Cov}\left( \log \pi, \frac{r - V^{EBPO}_q - \mu_{\hat{A}}}{\sigma_{\text{batch}}} \right) \right] \\
    &= \frac{\eta}{\sigma_{\text{batch}}} \mathbb{E}_{q} \left[ \text{Cov}( \log \pi, r - \underbrace{(V^{EBPO}_q + \mu_{\hat{A}})}_{\text{Constant w.r.t } o} ) \right]
\end{align*}
By the translation invariance of covariance ($\text{Cov}(X, Y - C) = \text{Cov}(X, Y)$), the baseline $V^{EBPO}_q$ and the batch mean $\mu_{\hat{A}}$ vanish from the covariance term:
\begin{equation*}
    \Delta H_{\text{EBPO}} = \frac{\eta}{\sigma_{\text{batch}}} \mathbb{E}_{q} [ \Omega_q ]
\end{equation*}
Assuming the covariance term $\Omega_q$ is independent of the group scale $\sigma_q$, the expected entropy reduction is governed by the effective scaling factors:
$$ \mathbb{E}_q \left[ \frac{1}{\sigma_q} \right] \quad \text{vs} \quad \frac{1}{\sigma_{\text{batch}}} $$

By the Law of Total Variance, the global batch variance $\sigma_{\text{batch}}^2$ decomposes exactly into the expected within-group variance and the between-group variance:
\begin{equation}
    \sigma_{\text{batch}}^2 = \mathbb{E}_q[\sigma_q^2] + \text{Var}_q(\mu_q)
\end{equation}
Given that the task distribution is heterogeneous (i.e., prompts have varying difficulty), we have $\text{Var}_q(\mu_q) > 0$. This implies a strict inequality for the second moments:
\begin{equation}
    \sigma_{\text{batch}} = \sqrt{\mathbb{E}_q[\sigma_q^2] + \text{Var}_q(\mu_q)} > \sqrt{\mathbb{E}_q[\sigma_q^2]} \label{step1}
\end{equation}

Note that for any random variable $X$, $\sqrt{\mathbb{E}[X^2]} \geq \mathbb{E}[X]$. Letting $X = \sigma_q$:
\begin{equation}
    \sqrt{\mathbb{E}_q[\sigma_q^2]} \geq \mathbb{E}_q[\sigma_q] \label{step2}
\end{equation}
Combining \eqref{step1} and \eqref{step2}:
\begin{equation*}
    \sigma_{\text{batch}} > \mathbb{E}_q[\sigma_q]
\end{equation*}

Consider the function $f(x) = 1/x$, which is strictly convex for $x > 0$. By Jensen's Inequality:
\begin{equation*}
    \mathbb{E}_q \left[ \frac{1}{\sigma_q} \right] \geq \frac{1}{\mathbb{E}_q[\sigma_q]}
\end{equation*}
Chaining the inequalities from above yields a strict ordering of the decay coefficients:
\begin{equation*}
    \mathbb{E}_q \left[ \frac{1}{\sigma_q} \right] \geq \frac{1}{\mathbb{E}_q[\sigma_q]} > \frac{1}{\sigma_{\text{batch}}}
\end{equation*}
Consequently, the entropy reduction for GRPO strictly upper-bounds that of EBPO:
\begin{equation*}
    \Delta H_{\text{GRPO}} = \eta \mathbb{E}_q \left[ \frac{\Omega_q}{\sigma_q} \right] > \eta \frac{1}{\sigma_{\text{batch}}} \mathbb{E}_q [\Omega_q] = \Delta H_{\text{EBPO}}
\end{equation*}
This proves that the global normalization in EBPO strictly enforces a more conservative policy update bound than GRPO, preventing the excessive entropy collapse associated with small-variance groups.
\end{proof}

\begin{proof}[Proof of Proposition~\ref{prop:clustering}]
We seek to minimize the expected squared error of the prior, defined as $J = \mathbb{E}_{q \sim \mathcal{D}} [(\hat{\mu}_{\text{prior}} - \theta_q)^2]$.

\textbf{Case 1: Random Shuffle (Topic-Agnostic)}
By Law of Large Numbers, the online estimator converges to the global expectation over the entire dataset. Thus, $\hat{\mu}_{\text{prior}} = \mu_{\text{glob}}$. The loss function becomes the total variance of the difficulty distribution:
\begin{equation*}
    J_{\text{shuffle}} = \mathbb{E}_{q} [(\mu_{\text{glob}} - \theta_q)^2] = \text{Var}(\theta_q)
\end{equation*}

\textbf{Case 2: Topic-Coherent (Topic-Conditional)}
In the idealized topic-coherent regime (assuming the Welford update is sufficiently fast to adapt to the topic shift), the estimator converges to the conditional expectation of the current cluster. For a prompt $q \in \mathcal{C}_k$, the prior is $\hat{\mu}_{\text{prior}} = \mu_k$. The loss function is the expected variance within each cluster:
\begin{equation*}
    J_{\text{coherent}} = \mathbb{E}_{k} \left[ \mathbb{E}_{q \in \mathcal{C}_k} [(\mu_k - \theta_q)^2] \right] = \mathbb{E}_k [\text{Var}(\theta_q \mid k)]
\end{equation*}

The total variance can be decomposed into the expected conditional variance (within-group) and the variance of the conditional expectations (between-group):
\begin{equation}
    \text{Var}(\theta_q) = \mathbb{E}_k [\text{Var}(\theta_q \mid k)] + \text{Var}_k (\mathbb{E}[\theta_q \mid k])
\end{equation}
Substituting our loss terms:
\begin{equation}
    J_{\text{shuffle}} = J_{\text{coherent}} + \text{Var}_k(\mu_k)
\end{equation}
Since the topics have distinct mean difficulties by definition, the between-topic variance is strictly positive: $\text{Var}_k(\mu_k) > 0$.
Therefore:
\begin{equation}
    J_{\text{coherent}} < J_{\text{shuffle}}
\end{equation}
The topic-coherent ordering strictly reduces the estimation error of the prior, thereby improving the accuracy of the EBPO baseline.
\end{proof}

\end{document}